\documentclass[letterpaper, 10 pt, conference]{ieeeconf}  

\IEEEoverridecommandlockouts                              
\usepackage{cite}
\usepackage{amsmath,amssymb,amsfonts}
\usepackage{algorithmic}
\usepackage{graphicx}
\usepackage{textcomp}
\usepackage{setspace}
\usepackage{booktabs}
\usepackage{multirow}

\usepackage{epsfig}
\usepackage{times} 
\usepackage{svg}
\usepackage[linesnumbered,ruled,vlined]{algorithm2e} 

\usepackage[nolist, nohyperlinks]{acronym}

\usepackage{bm}
\usepackage{breqn}
\usepackage{amsmath}
\usepackage{amssymb}
\usepackage{tabularx}

\usepackage{amsthm}

\newcommand{\Znn}{\mathbb{Z}_{\geq 0}}
\newcommand{\Zp}{\mathbb{Z}_{>0}}
\newcommand{\R}{\mathbb{R}}

\newcommand{\xbf}{\mathbf{x}}
\newcommand{\ubf}{\mathbf{u}}

\newcommand{\Ccal}{\mathcal{C}}

\newcommand{\Ncal}{\mathcal{N}}
\newcommand{\Ocal}{\mathcal{O}}

\newcommand{\Rcal}{\mathcal{R}}

\newcommand{\Tcal}{\mathcal{T}}
\newcommand{\Ucal}{\mathcal{U}}

\newcommand{\Xcal}{\mathcal{X}}
\newcommand{\Ycal}{\mathcal{Y}}

\newcommand{\eqn}[1]{\begin{align} #1 \end{align}}
\newcommand{\eqnN}[1]{\begin{align*} #1 \end{align*}}
\newcommand{\seqn}[2][]{
\begin{subequations}
    #1
\begin{align} #2 \end{align}
\end{subequations}
}

\newcommand{\argmin}[1]{\underset{\substack{#1}}{\text{argmin}}}

\theoremstyle{plain}

\theoremstyle{definition}
\newtheorem{assumption}{Assumption}
\newtheorem{remark}{Remark}
\newtheorem{prop}{Proposition}

\makeatletter
\let\NAT@parse\undefined
\makeatother
\usepackage{hyperref}
\hypersetup{
colorlinks, 
    linkcolor={red!50!black},
    citecolor={blue!50!black},
    urlcolor={blue!80!black}
}
\usepackage{cleveref}

\crefformat{problem}{Problem~#2#1#3}
\crefformat{assumption}{Assumption~#2#1#3}
\crefformat{prop}{Proposition~#2#1#3}
\usepackage[table]{xcolor} 
\definecolor{d4ormcolor2}{RGB}{248,235,245} 
\definecolor{mbdcolor}{RGB}{255,239,220} 
\definecolor{d4ormcolor}{RGB}{232,247,232} 
\definecolor{ourscolor}{RGB}{225,245,245} 
\definecolor{distcolor}{RGB}{240,240,240} 
\definecolor{oomcolor}{RGB}{200,40,40}

\title{\LARGE \bf Distributed Model-Based Diffusion for Scalable Multi-Robot \\Trajectory Optimization
}
\author{
Haejoon Lee, Xinyi Wang, Taekyung Kim, and Dimitra Panagou
\thanks{All authors are with the Robotics Department, University of Michigan, Ann Arbor, MI, USA
        {\tt \{haejoonl, xinywa, taekyung, dpanagou\}@umich.edu}}
\thanks{$^a$Project Page: \href{https://distributed-mbd.github.io/dmbd.github.io/}{https://distributed-mbd.github.io/dmbd.github.io/}
}}

\begin{document}

\maketitle
\thispagestyle{empty}
\pagestyle{empty}

\begin{abstract}
Trajectory optimization for multi-robot systems remains a critical challenge, particularly when navigating highly non-convex, non-linear, and non-differentiable environments. While Model-Based Diffusion (MBD) has recently emerged as a promising sampling-based optimization paradigm for single-robot trajectory generation, extending it to multi-robot systems results in a centralized, high-dimensional inference problem that (i) suffers from poor sample efficiency due to the curse of dimensionality and (ii) requires global access to all robots' dynamics, constraints, and objectives. To address this, we propose Distributed Model-Based Diffusion (DMBD), a distributed server-robot method that decomposes the reverse diffusion process into local conditional reverse diffusion processes. This decomposition enables each robot to iteratively perform denoising independently within its own control subspace while conditioning on the current trajectory estimates of the other robots that are aggregated and broadcast by the server. Extensive simulations in goal swapping, multi-floor coverage, parking, and rush-hour scenarios demonstrate that DMBD achieves strong scalability, solving many challenging coordination tasks with sub-second computation time and outperforming existing baselines. \href{https://distributed-mbd.github.io/dmbd.github.io/}{[Project Page]}$^a$
\end{abstract}

\section{Introduction}

Trajectory optimization for collision-free multi-robot coordination in shared workspaces remains a critical challenge. A popular approach is gradient-based optimization, which uses cost gradient information to seek optimal solutions~\cite{nocedal2006numerical, van2017distributed, marcucci2023motion, shorinwa2024distributed}. However, these methods often struggle in highly nonconvex or non-differentiable settings, where they are susceptible to poor local minima~\cite{marcucci2023motion} and incur substantial computational costs as the team size increases~\cite{hopcroft1984complexity}.

To circumvent the limitations of gradient-based optimization methods, sampling-based optimization (SBO) methods have emerged as a powerful alternative. By casting the trajectory planning problem as a probabilistic inference problem over the trajectory space, SBO planners have proven highly successful in navigating complex, non-smooth, and non-convex environments~\cite{rubinstein1999cross,williams2018information, pan2024model, xue2025full}. Representative examples include the Cross Entropy Method (CEM)~\cite{rubinstein1999cross,botev2013cross}, Model Predictive Path Integral (MPPI)~\cite{williams2018information,kim2022smooth}, and variational MPCs~\cite{jiang2024distributed,pacelli2026sampling}.

Recently, Model-Based Diffusion (MBD) has emerged as a promising SBO method inspired by generative denoising diffusion probabilistic models~\cite{ho2020denoising, pan2024model}. While algorithmically similar, it departs from other SBO approaches, as MBD leverages the denoising diffusion framework to optimize over a multi-modal trajectory distribution via iterative denoising process. Furthermore, in contrast to other diffusion-based planners that rely on score networks learned from offline datasets~\cite{janner2022planning, zhong2023guided, liang2026discrete}, MBD directly exploits known system dynamics and task objectives during inference, enabling \emph{learning-free} generation of diverse, low-cost trajectories.

Since its introduction, several extensions of MBD have been proposed. For example, gradual constraint enforcement was introduced in~\cite{mishra2025eb} to progressively improve solution quality throughout the denoising process, while strict safety guarantees were incorporated directly into the denoising procedure in~\cite{kim2025safe}. Adaptive importance sampling via CEM was employed in~\cite{golembeski2025importance} to improve the performance of MBD. Related diffusion-inspired denoising schemes were adopted in~\cite{xue2025full} to present a real-time controller.

MBD's ability to optimize over multi-modal trajectory distributions is appealing for multi-robot planning, where deadlocks often arise from poor local minima. However, despite its success in single-agent settings, extending MBD to multi-robot systems remains challenging for two reasons.

First, MBD suffers from the \emph{curse of dimensionality} as the number of robots increases. Because the dimension of the joint trajectory space grows with the number of robots, MBD, like other SBO methods, experiences exponentially decreasing sampling efficiency and optimization performance in large-scale multi-robot settings. D4orm~\cite{zhang2025d4orm} mitigates this issue through iterative deformation updates, but its reverse diffusion process remains centralized, requiring a single processor to sample and denoise trajectories in the joint trajectory space. A recent work~\cite{he2026motion} combines locally executed MBD planning with centralized deconfliction, which may require an increasing number of rounds of deconfliction and subsequent MBD replanning as the number of robots grows and the environment becomes more congested.

Second, MBD assumes \emph{centralized access to all objectives, constraints, and robot dynamics} to perform denoising. In practical multi-robot systems, particularly heterogeneous teams, robots often possess distinct objectives, dynamics, and local constraints. Aggregating this information can incur substantial communication overhead and/or places a significant computational and memory burden on a single machine, further limiting scalability.


To address these challenges, this paper introduces Distributed Model-Based Diffusion (DMBD), a distributed, server-robot algorithm for multi-robot trajectory optimization. DMBD decomposes the global trajectory optimization problem into local inference subproblems that are executed independently and in parallel by each robot. Rather than estimating a global score function over the full joint trajectory distribution, each robot estimates a local conditional score function to denoise its own trajectory using only local information. At each denoising step, this local score is conditioned on the current trajectory estimates of the other robots, which are aggregated and broadcast by the server.

In contrast to existing distributed sampling-based approaches~\cite{jiang2024distributed,wan2021cooperative}, which require agents to exchange multiple trajectory samples and therefore incur substantial communication overhead, each robot in DMBD transmits only a single trajectory per denoising step. Unlike~\cite{streichenberg2023multi}, our method does not require robots to know other robots' local objective functions. Furthermore, unlike gradient-based approaches such as~\cite{pavlasek2024stein}, DMBD is entirely zeroth-order and does not assume differentiability. Consequently, DMBD achieves scalable multi-robot coordination while retaining the efficiency and multi-modal trajectory generation of MBD.

Our contributions are as follows:
\begin{itemize}
\item We propose \textbf{Distributed Model-Based Diffusion (DMBD)}, a server-coordinated distributed MBD that decomposes multi-robot trajectory optimization into local conditional diffusion processes. Rather than estimating a global score function over the joint trajectory, each robot independently estimates a local conditional score function and denoises its own trajectory conditioned on the trajectories of other robots. 

\item We establish the relationship between MBD and DMBD by theoretically bounding the discrepancy between the global and the local conditional score functions.

\item We validate DMBD through extensive simulations across diverse multi-robot scenarios, including goal swapping, multi-floor coverage, parking, and rush-hour tasks, demonstrating superior scalability and consistent performance gains over existing baselines.
\end{itemize}

\section{Preliminaries}
We denote the non-negative and positive integers by $\mathbb{Z}{\geq 0}$ and $\mathbb{Z}{>0}$, respectively, and the $p\times p$ identity matrix by $I_p$.

Consider a multi-robot system of $N$ robots connected to a server. Each robot $k\in \Rcal:=\{1,\dots,N\}$ is modeled by 
\eqn{\label{eq:single_dynamics}
\xbf^k_{t+1}= f_k(\xbf_t^k, \ubf^k_t),}
where $\xbf_t^k\in \Xcal_k\subset \R^{n_k}$ and $\ubf_t^k\in \Ucal_k\subset \R^{m_k}$ are the state and control input at time step $t\in \Znn$, respectively. We denote $n=\sum_{k\in \Rcal}n_k$ and $m=\sum_{k\in\Rcal}m_k$. The function $f_k:\Xcal_k \times \Ucal_k \to \Xcal_k$ represents the dynamics of robot $k$. We stack the states and control inputs of robots at time $t$ into:
\eqn{\xbf_t = \begin{bmatrix}
    {\xbf^1_t}^\top & \cdots & {\xbf^N_t}^\top
\end{bmatrix}^\top\in \Xcal := \prod_{k=1}^N\Xcal_k, \\
\ubf_t = \begin{bmatrix}
    {\ubf^1_t}^\top & \cdots & {\ubf^N_t}^\top
\end{bmatrix}^\top \in \Ucal := \prod_{k=1}^N\Ucal_k.
}

Then, the global system can be written as:
\eqn{\label{eq:total_dynamics}
\xbf_{t+1}= F(\xbf_t, \ubf_t),}
where $F:\Xcal\times \Ucal \to \Xcal$ is the dynamics of all robots. 

For a planning horizon $T\in\Zp$, we define the control trajectory of all robots as
\eqnN{Y:= \begin{bmatrix}
    Y_1^\top & \cdots & Y_N^\top
\end{bmatrix}^\top \in \Ucal^T := \prod_{k=1}^N \Ucal_k^T \subset \R^{mT},}
where \eqnN{
\label{eq:local_traj}
Y_k :=
\begin{bmatrix}
(\ubf_0^k)^\top \
\cdots \ 
(\ubf_{T-1}^k)^\top
\end{bmatrix}^\top \in \Ucal^T_k := \prod_{q=1}^T \Ucal_k\subset \R^{m_kT}
}
denotes the local control trajectory of robot $k$. We denote $Y_{-k}$ as the collection of control trajectories of all robots except robot $k$.

Given the initial local state $\xbf_0^k$ and local control trajectory $Y_k$, the state trajectory of robot $k$ is uniquely determined through the recursive rollout of~\eqref{eq:single_dynamics}, which we denote by
\eqnN{\tau_{k}:=\tau_k(Y_k) = (\xbf_0^k, f_k(\xbf_0^k, \ubf_0^k), f_k(f_k(\xbf_0^k, \ubf_0^k), \ubf_1^k), \cdots).}

Similarly, given the initial global state $\xbf_0$ and global control trajectory $Y$, the global state trajectory is determined through the recursive rollout of~\eqref{eq:total_dynamics}, which we denote by $\tau:=\tau(Y)$. We further denote by $\tau_{-k}:=\tau_{-k}(Y_{-k})$ the collection of state trajectories of all robots except robot $k$.

Each robot $k$ aims to minimize a local objective function 
\eqnN{
J_k(Y_k):=J_k(Y_k, \tau_k(Y_k)) = l_{T,k}(\xbf_T^k) + \sum_{t=0}^{T-1} l_{t,k}(\xbf_t^k, \ubf_t^k),
}
where $l_{T,k}:\Xcal_k \to \R$ and $l_{t,k}:\Xcal_k \times \Ucal_k \to \R$ are terminal and stage costs.

In addition, robot $k$ is subject to a set of local constraints $\Ccal_k$. A constraint $p\in \Ccal_k$ is given as
\eqnN{
g_{k,p}(Y) := g_{k,p}(\tau_k, \tau_{-k}),
}
which is satisfied when non-positive. Constraint $p$ may depend not only on its own state trajectory but also on the state trajectories of other robots. We denote the index set of robots involved in this constraint including robot $k$ by $\Rcal_{k,p} \subseteq \Rcal$. We assume that whenever robots are involved in the same constraint function, they all share the same knowledge of the functional form and parameters of the constraints. That is,
\begin{assumption}
\label{assum:same_constraint}
For any $k\in\Rcal$, $p\in\Ccal_k$, and $j\in\Rcal_{k,p}$, the coupled constraint is symmetric with respect to the participating robots, i.e., $\exists p' \in \Ccal_j$ such that
$\Rcal_{j,p'} = \Rcal_{k,p}$ and $
g_{j,p'}(Y)=g_{k,p}(Y)$, $\forall Y \in \Ucal^T$.
\end{assumption}

We make no assumptions regarding the convexity or continuity of $J_k$ and $g_{k,p}$ to reflect the complexities of real-world robotic tasks. Under these settings, the centralized trajectory optimization problem is given by
\seqn[\label{eq:optimization}]{
\argmin{Y} \ & J(Y)= \sum_{k=1}^N J_k(Y_k) \\
\text{s.t.}\quad  &\xbf_{t+1}= F(\xbf_t, \ubf_t), \quad \forall t \in \{0, \dots, T-1\}, \\
& \ubf_t \in \Ucal, \quad \forall t \in \{0, \dots, T-1\}, \\
& g_{k,p}(Y) \leq 0, \ \forall p \in \Ccal_k, \forall k \in \Rcal.}

Traditionally, solving~\eqref{eq:optimization} requires solving a nonlinear program. However, such approaches often converge to poor local minima and may even fail to converge when applied to non-convex and non-smooth optimization problems.

\subsection{Sampling-Based Trajectory Optimization}

An alternative paradigm that has recently gained significant attention is sampling-based optimization (SBO) which essentially casts the problem as an inference problem over the trajectory space~\cite{williams2018information, pan2024model}. We first reformulate the centralized problem~\eqref{eq:optimization} through the lens of centralized probabilistic inference. Note that we redefine the inference problem into the distributed form later in~\Cref{sec:method}.

We define a target distribution $p_0(Y)$ over the space of trajectories: 
\eqn{\label{eq:target_dist}
p_0(Y) \propto \exp(-\phi(Y)/\lambda),}
with temperature $\lambda >0$, where we define global cost function
\eqn{\phi(Y):= J(Y) + \sum_{k\in\Rcal, p\in \Ccal_k} \frac {1}{|\Rcal_{k,p}|} g_{k,p}(Y).\label{eq:actual_global_cost}}

This formulation converts the constrained problem~\eqref{eq:optimization} into an unconstrained probabilistic formulation by defining $p_0$ such that lower-cost trajectories receive higher probability mass while penalizing constraint violations. Hence, trajectories with smaller cost $\phi(Y)$ are more likely under $p_0$.

\subsection{Model-Based Diffusion (MBD)}
In a centralized setting, since $J$, $g_{k,p}$, and $F$ are all known, we can evaluate the probability $p_0(Y)$ of a trajectory $Y\in \Ucal^T$. Nevertheless, because $\phi$ can be an arbitrary function, sampling directly from $p_0$ is generally intractable. Thus, MBD iteratively refines samples starting from a Gaussian distribution~\cite{pan2024model}. MBD is characterized by two processes:

\textbf{Forward (noising) process:}
The forward process gradually transforms a clean trajectory $Y^{(0)} \sim p_0$ into white noise over $M$ discrete steps. Given a variance schedule $\{\beta_1, \dots, \beta_M\}$, we define $\alpha_i := 1 - \beta_i$ and $\bar{\alpha}_i := \prod_{j=1}^i \alpha_j$. The noising step at iteration $i$ is defined by:
\begin{equation}
Y^{(i)} = \sqrt{\alpha_i} Y^{(i-1)} + \sqrt{1 - \alpha_i} \epsilon_i, \quad \epsilon_i \sim \mathcal{N}(0, I_{mT}).
\end{equation}
By utilizing the property of Gaussian sums, we can directly sample $Y^{(i)}$ from $Y^{(0)}$ as:
\eqn{\label{eq:p_i_definition}
p_{i\mid0}(\cdot \mid Y^{(0)}) = \Ncal(\sqrt{\bar{{\alpha}}_i} Y^{(0)}, (1 - \bar{\alpha}_i)I_{mT}).
}
As $M \to \infty$, the distribution $p_M(Y^{(M)})$ approaches a Gaussian distribution $\mathcal{N}(0, I_{mT})$.

\textbf{Reverse (denoising) process:}
The reverse process aims to recover a low-cost trajectory by starting from $Y^{(M)} \sim \mathcal{N}(0, I_{mT})$ and moving back toward the target distribution defined by $p_0$. Unlike traditional diffusion models~\cite{janner2022planning, zhong2023guided, liang2026discrete} that learn score functions via neural networks, MBD utilizes knowledge of $J$, $F$, and $g_{k,p}$ to perform a Monte Carlo Score Ascent~\cite{pan2024model}. At each denoising step $i\in\{M,\dots, 1\}$, the following update is performed:
\begin{equation}
\label{eq:global_update}
Y^{(i-1)} = \frac{1}{\sqrt{\alpha_i}} \left( Y^{(i)} + (1 - \bar{\alpha}_i) \nabla_{Y^{(i)}} \log p_i(Y^{(i)}) \right).
\end{equation}
The score function $\nabla_{Y^{(i)}}  \log p_i(Y^{(i)})$ is estimated with $S$ candidate trajectories $\{\Ycal^{(i)}_{s}\}_{s=1}^S$ sampled from the distribution $\Ncal\left(\frac{Y^{(i)}}{\sqrt{\bar{\alpha}_i} }, \left(\frac 1 {\bar{\alpha}_i} -1\right)I_{mT}\right)$ as below:
\eqn{\label{eq:gradient}
\nabla_{Y^{(i)}}  \log p_i(Y^{(i)})\approx - \frac{Y^{(i)}}{1-\bar{\alpha}_i} + \frac {\sqrt{\bar{\alpha}_i}}{1-\bar{\alpha}_i}\hat{Y}^{(i)},}
where $\hat{Y}^{(i)}$ is a weighted average of the sampled candidate
control trajectories $\{\Ycal^{(i)}_{s}\}_{s=1}^S$, i.e., 
\eqn{\hat{Y}^{(i)} = \frac{\sum_{s=1}^S \Ycal^{(i)}_s p_0(\Ycal^{(i)}_s)}{\sum_{s=1}^S p_0(\Ycal^{(i)}_s)}.}

This estimated score~\eqref{eq:gradient} provides the steepest direction that shifts the current estimate $Y^{(i)}$ toward high-density regions in $p_0$ and, consequently, lower-cost trajectories.

In essence, at each denoising step, MBD exploits explicit \emph{knowledge of the model}, including dynamics, objectives, and constraints, to iteratively sample, evaluate, and refine candidate trajectories toward the optimal solution of $\phi$, without requiring any learning process. For further details on its convergence properties and the underlying intuition, we refer readers to~\cite{yi2026global}.

\begin{figure*}[t]
    \centering
    \includegraphics[width=0.9\linewidth]{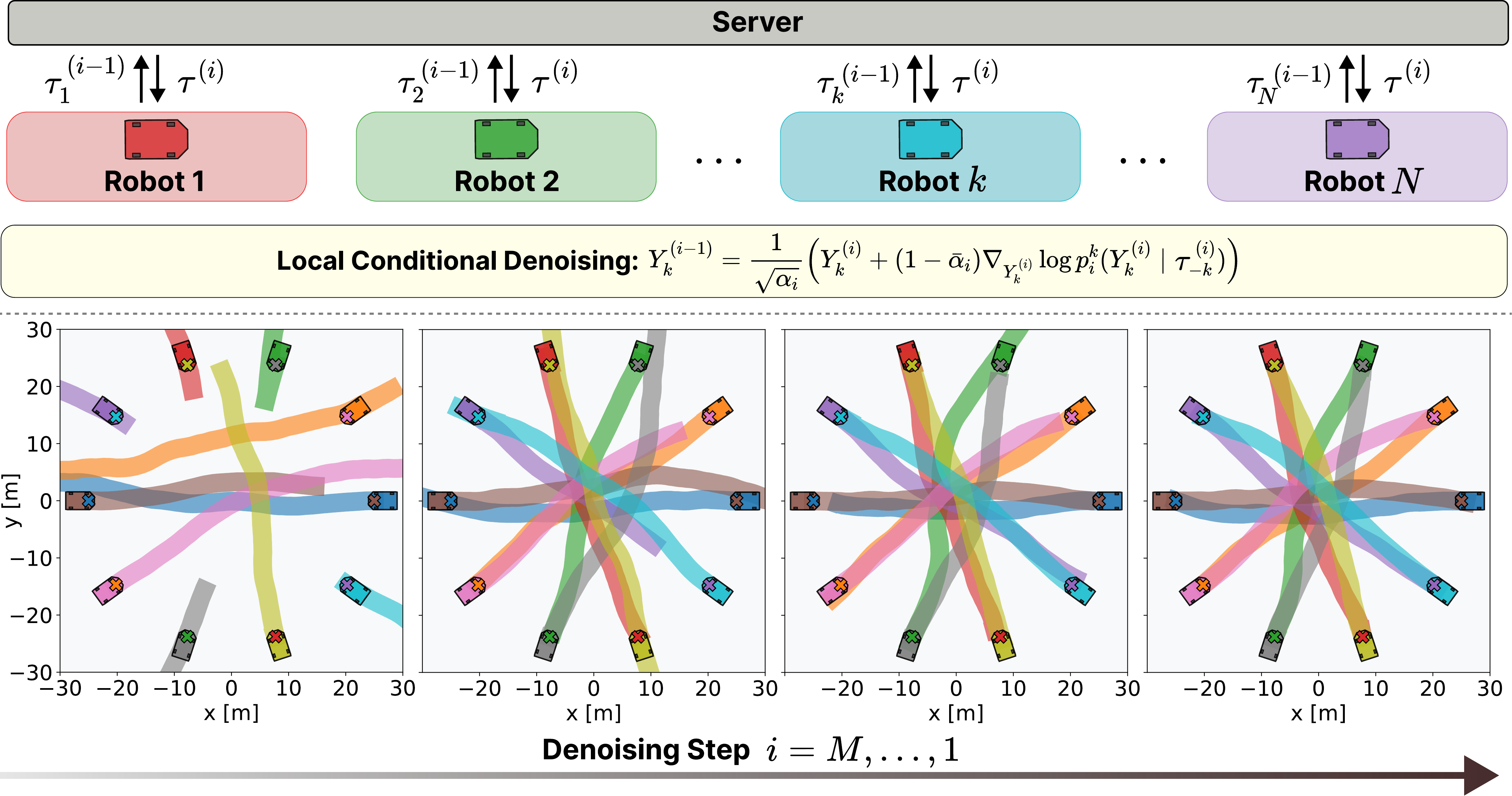}
    \caption{Visualization of Distributed Model-Based Diffusion (DMBD). For every denoising step, each robot independently (i) performs local conditional denoising updates on its control trajectory conditioned on the state trajectories of other robots aggregated and broadcast by the server and (ii) transmits its updated state trajectory back to the server, which aggregates and rebroadcasts them.}
    \label{fig:hero}
\vspace{-10pt}
\end{figure*}

\section{Distributed Model-Based Diffusion (DMBD)}
\label{sec:method}
While MBD provides an efficient mechanism for sampling low-cost trajectories under $p_0$, it has two key limitations. First, it assumes access to the global objective $J$, global dynamics $F$, and all constraint functions $g_{k,p}$. In a multi-robot setting, however, each local objective $J_k$, dynamics $f_k$, and constraints $g_{k,p}$ may only be available to robot $k\in\Rcal$. Second, MBD scales poorly with the number of robots, since the dimensionality of the joint control space grows with $N$, reducing sampling efficiency and thus performance.

To address these issues, we propose \emph{Distributed Model-Based Diffusion} (DMBD), a server-robot framework that decomposes the joint reverse diffusion process into distributed conditional denoising processes. DMBD shares the same noising process as vanilla MBD but differs in its denoising process. At each denoising step, robot $k$ locally updates its control trajectory $Y_k \in \Ucal_k^T$ using its local objective, dynamics, and constraints, while being conditioned on the current state trajectory estimates of other robots aggregated and broadcast by the server. Each robot then transmits its updated state trajectory back to the server, which aggregates and rebroadcasts them for the next step (see~\Cref{fig:hero}).

We first define a local conditional target distribution for each robot $k\in\Rcal$:
\eqn{\label{eq:local_dist}
p_0^k\left(Y_k\mid \tau_{-k}\right) \propto \exp(-\phi_k(Y_k;\tau_{-k})/\lambda),}
where $\lambda>0$ and the local cost function 
\eqn{\phi_k(Y_k;\tau_{-k}):= J_k(Y_k) + \sum_{p\in \Ccal_k} g_{k,p}(\tau_{k}, \tau_{-k}).\label{eq:actual_local_cost}}

DMBD begins with each robot $k\in\Rcal$ sampling $Y^{(M)}_k\sim \mathcal{N}(0, I_{m_kT})$ and sending its state trajectory $\tau_k^{(M)}:=\tau_k(Y_k^{(M)})$ to the server, which aggregates into
\eqn{
\label{eq:aggregated}
\tau^{(M)} :=\tau(Y^{(M)}) = \begin{bmatrix}
   \tau_1^{(M)} & \cdots &  \tau_N^{(M)} 
\end{bmatrix}
}
and broadcasts~\eqref{eq:aggregated} back to each robot. Upon receiving $\tau^{(M)}$, each robot $k$ performs the local conditional denoising process over $M$ steps.

\textbf{Local conditional reverse (denoising) process:} At each step $i =M, \dots, 1$, each robot locally samples and refines its control input trajectory $Y^{(i)}_k$ with respect to its local cost $\phi_k$~\eqref{eq:actual_local_cost} while being conditioned on the shared state trajectories of other robots $\tau_{-k}^{(i)}:=\tau_{-k}(Y^{(i)}_{-k})$. The  denoising update is given by 
\eqn{
\label{eq:local_update}
Y^{(i-1)}_k = \frac{Y^{(i)}_k}{\sqrt{\alpha_i}} + \frac{1 - \bar{\alpha}_i}{\sqrt{\alpha_i}} \nabla_{Y^{(i)}_k} \log p_i^k\left(Y^{(i)}_k \mid \tau_{-k}^{(i)}\right).}

Unlike the centralized update~\eqref{eq:global_update}, each robot $k$ estimates its local conditional score function $\nabla_{Y^{(i)}_k} \log p_i^k\left(Y^{(i)}_k \mid \tau_{-k}^{(i)}\right)$. The score function is estimated by sampling and evaluating $S_k$ candidate control trajectories $\{\Ycal^{(i)}_{k,s}\}_{s=1}^{S_k}$ of its own from the distribution $\Ncal\left(\frac{Y^{(i)}_k}{\sqrt{\bar{\alpha}_i} }, \left(\frac 1 {\bar{\alpha}_i} -1\right)I_{m_kT}\right)$ while fixing $\tau_{-k}^{(i)}$. Thus, we have
\eqn{\label{eq:local_gradient}
\nabla_{Y^{(i)}_k} \log p_i^k\left(Y^{(i)}_k \mid \tau_{-k}^{(i)}\right)\approx - \frac{Y^{(i)}_k}{1-\bar{\alpha}_i} + \frac {\sqrt{\bar{\alpha}_i}}{1-\bar{\alpha}_i}\hat{Y}^{(i)}_k,}
where 
\eqn{\hat{Y}^{(i)}_k = \frac{\sum_{s=1}^{S_k} \Ycal^{(i)}_{k,s} p_0^k\left(\Ycal^{(i)}_{k,s} \mid \tau_{-k}^{(i)}\right)}{\sum_{s=1}^{S_k} p_0^k\left(\Ycal^{(i)}_{k,s} \mid \tau_{-k}^{(i)}\right)}.\label{eq:local_weighted}}
Intuitively, $\hat{Y}^{(i)}_k$ is a locally weighted average of candidate control trajectories for robot $k$, where the weights depend on the conditional likelihood induced by the local target distribution $p_0^k(\cdot \mid \tau_{-k}^{(i)})$.

After the update, the robots obtain their new state trajectories $\tau_k^{(i-1)}$ through rollout and transmit to the server, which aggregates them into the global trajectory $\tau^{(i-1)}$ as in~\eqref{eq:aggregated}, and broadcasts it back to all robots. This cycle iterates until $i=1$, as detailed in~\Cref{alg:dmbd}.

\begin{algorithm}[t]
\SetKwInOut{Inputs}{Inputs}
\SetKwInOut{Outputs}{Outputs}
\small
\caption{Distributed MBD (DMBD)}
\label{alg:dmbd}

\Inputs{noise schedule $\{\bar{\alpha}_i\}_{i=1}^M$, denoising steps $M$, numbers of samples $\{S_k\}_{k=1}^N$, temperature $\lambda>0$}
\Outputs{optimized trajectories $\{Y_k^{(0)}\}_{k=1}^N$}

\tcp{Trajectory Initialization}
\For{$k = 1,\dots, N$}{
Sample $Y^{(M)}_k \sim \mathcal{N}(0, I_{m_kT})$

Send local state trajectory $\tau_k^{(M)}$ to server
}

\For{$i=M,\dots,1$}{
    \tcp{Server Aggregation}
    Aggregate and broadcast $
    \tau^{(i)}=
    \big[
    \tau_1^{(i)}, \dots, \tau_N^{(i)}
    \big]$
    
\For{$k = 1,\dots, N$}{
    \tcp{Executed in parallel by robot $k$}

    Receive state trajectories of other robots $\tau^{(i)}_{-k}$
    
    Sample $\{\mathcal{Y}^{(i)}_{k,s}\}_{s=1}^{S_k}
    \sim
    \mathcal{N}
    \left(
    \frac{Y^{(i)}_k}{\sqrt{\bar{\alpha}_{i}}},
    \left(
    \frac{1}{\bar{\alpha}_{i}}-1
    \right)I_{m_kT}
    \right)$

    Compute $\hat{Y}^{(i)}_k$ according to~\eqref{eq:local_weighted}

    Estimate
    $\nabla_{Y_k^{(i)}} \log p_i^k\left(Y_k^{(i)} \mid \tau_{-k}^{(i)} \right)$ via~\eqref{eq:local_gradient}

    Update $Y_k^{(i-1)}$ with~\eqref{eq:local_update}

    Send local state trajectory $\tau_k^{(i-1)}$ to the server
}
}
\Return{$\{Y_k^{(0)}\}_{k\in \Rcal}$}
\end{algorithm}

\begin{remark}
Robot $k$'s objective and dynamics are known only to robot $k$, and coupled
constraints are shared only among their participating robots (\Cref{assum:same_constraint}). Each robot
thus requires neither access to nor storage of global system information. The server also performs no optimization: it only aggregates and rebroadcasts trajectories, following a server-robot framework common in warehouse and delivery systems~\cite{wurman2008coordinating,gielis2022critical}.
\end{remark}


By distributing the denoising computation across robots, DMBD decomposes the global high-dimensional inference problem over the joint trajectory space of dimension $\sum_{k\in\Rcal} m_kT$ into independent local problems of dimension $m_kT$. Thus, DMBD enjoys improved sample efficiency and scalability with respect to $N$.

While DMBD decomposes the centralized reverse diffusion into local conditional processes, coupled constraints introduce dependencies among robots that alter how their trajectories are denoised, resulting in different reverse dynamics. Since $p_i(Y)$ is defined over the joint trajectory space, the centralized and local conditional scores generally differ. We therefore bound this discrepancy as follows:

\begin{prop}
\label{prop:score_bound} Let~\Cref{assum:same_constraint} hold. Suppose that, for each $k\in\Rcal$, the local denoiser 
$\mu_{k,i}(Y_k^{(i)},z) := \mathbb E[Y_k^{(0)}\mid Y_k^{(i)}, \tau_{-k}^{(0)}=z]$
is $H_k$-Lipschitz in $z$, i.e.,
$\|\mu_{k,i}(Y_k^{(i)},z) - \mu_{k,i}(Y_k^{(i)},z')\| \le H_k \|z-z'\|$ for all $Y_k^{(i)}, z, z'$ and $i\in \{1,\dots, M\}$. Then
for all denoising steps $i\in\{M,\dots,1\}$,
\eqn{
\label{eq:score_bound}
\left\|
\nabla_{Y_k^{(i)}}\log p_i(Y^{(i)})
-
\nabla_{Y_k^{(i)}}\log p_i^k\left(Y_k^{(i)} \mid \tau_{-k}^{(i)}\right)
\right\| \nonumber \leq  \\
\frac{\sqrt{\bar\alpha_i}}{1-\bar\alpha_i}\, H_k\,
\mathbb E\left[\, \|\tau_{-k}^{(0)} - \tau_{-k}^{(i)} \| \middle| \ Y^{(i)}  \right].
}
\end{prop}

\begin{proof}
For simplicity, we drop the notation $(i)$ for the rest of the proof.
By~\Cref{assum:same_constraint}, we can rewrite $\phi$~\eqref{eq:actual_global_cost} into
\eqnN{\phi(Y) &= J(Y) + \sum_{k\in\Rcal, p\in \Ccal_k} \frac {1}{|\Rcal_{k,p}|} g_{k,p}(Y)\\ 
& =\underbrace{J_k(Y_k) + \sum_{ p\in \Ccal_k} g_{k,p}(\tau_k, \tau_{-k})}_{\phi_k(Y_k;\tau_{-k})}+ \phi_{-k}(Y_{-k})}
where $\phi_{-k}$ is a function that does not depend on $Y_k$. Because $\phi_{-k}(Y_{-k})$ is constant with respect to $Y_k$, we get 
\eqn{p_0(Y_k^{(0)}\mid Y_{-k}^{(0)})= 
p_0^k(Y_k^{(0)}
\mid \tau_{-k}^{(0)}). \label{eq:factor1}
}

Under the forward diffusion process, the diffusion noises are independent across robots. Thus, we get
\eqn{p_{i\mid 0}(Y\mid Y^{(0)}) = \prod_{k=1}^N p_{i\mid 0}^k(Y_k\mid Y^{(0)}_k) \label{eq:factor2}.
}

By Tweedie's formula, we get:
\seqn[\label{eq:tweedie}]{
\nabla_{Y_k}\log p_i(Y) & = \frac{\sqrt{\bar\alpha_i}\,\mathbb E[Y_k^{(0)}\mid Y] - Y_k}{1-\bar\alpha_i}, \\
\nabla_{Y_k}\log p_i^k(Y_k\mid \tau_{-k}) & = \frac{\sqrt{\bar\alpha_i}\,\mu_{k,i}(Y_k,\tau_{-k}) - Y_k}{1-\bar\alpha_i}.
}
From~\eqref{eq:tweedie}, we get 
\eqn{\label{eq:intermediate}
\left\|
\nabla_{Y_k}\log p_i(Y)
-
\nabla_{Y_k}\log p_i^k\left(Y_k \mid \tau_{-k}\right)
\right\| \nonumber = \\
\frac{\sqrt{\bar\alpha_i}}{1-\bar\alpha_i}\left\|\mathbb E[Y_k^{(0)}\mid Y] - \mu_{k,i}(Y_k,\tau_{-k})\right\|.}

By~\eqref{eq:factor2}, $Y_{-k}$ is generated from $Y_{-k}^{(0)}$ through independent
forward noise, so $Y_{-k}$ is conditionally independent of $Y_k^{(0)}$ given
$Y_k$ and $Y_{-k}^{(0)}$. Also, by~\eqref{eq:factor1}, $Y_k^{(0)}$ depends on
$Y_{-k}^{(0)}$ only through $\tau_{-k}^{(0)}$. Then, the tower property gives
\seqn[\label{eq:factor3}]{
\mathbb E[Y_k^{(0)}\mid Y]
& = \mathbb E\!\left[
\mathbb E[
Y_k^{(0)}
\mid
Y_k, Y_{-k}, Y_{-k}^{(0)}
]
\mid
Y
\right] \\ & = 
\mathbb E\!\left[
\mathbb E[
Y_k^{(0)}
\mid
Y_k,\tau_{-k}^{(0)}
]
\mid
Y
\right] \\
& =
\mathbb{E}\left[
\mu_{k,i}(Y_k,\tau_{-k}^{(0)})
\mid Y
\right].
}

Applying~\eqref{eq:factor3}, Jensen's inequality (with the convexity of the norm), and Lipschitz continuity of $\mu_{k,i}$, we get
\eqn{
\Big\|\mathbb{E}[Y_k^{(0)}\mid Y] & -\mu_{k,i}(Y_k,\tau_{-k})\Big\| \label{eq:intermediate2}\\ &= \left\|\mathbb{E} \Big[ \mu_{k,i}(Y_k,\tau_{-k}^{(0)})\mid Y\Big] -\mu_{k,i}(Y_k,\tau_{-k}) \right\| \nonumber \\ &= \left\|\mathbb{E} \Big[ \mu_{k,i}(Y_k,\tau_{-k}^{(0)})  -\mu_{k,i}(Y_k,\tau_{-k}) \mid Y \Big]\right\| \nonumber \\ &\leq \mathbb{E} \Big[ \left\|\mu_{k,i}(Y_k,\tau_{-k}^{(0)})  -\mu_{k,i}(Y_k,\tau_{-k})\right\|  \mid Y \Big]\nonumber \\ 
& \leq H_k \mathbb{E}\Big[ \big\|\tau_{-k}^{(0)}-\tau_{-k} \big\|\mid Y \Big] \nonumber.}
Combining~\eqref{eq:intermediate} with~\eqref{eq:intermediate2}, we get~\eqref{eq:score_bound}.
\end{proof}

\Cref{prop:score_bound} isolates two sources of discrepancy introduced by the decomposition. The term $\mathbb{E}[\|\tau_{-k}^{(0)}-\tau_{-k}^{(i)}\|  \mid Y^{(i)}]$ measures the discrepancy between the other robots' current estimates and their fully denoised counterparts, which decreases as the denoised trajectories converge. Meanwhile, $H_k$ measures the sensitivity of robot $k$'s denoiser to changes in those trajectories. Thus, the discrepancy vanishes as either term approaches zero. In particular, if all robots have independent constraints, i.e., $\Rcal_{k,p}=\{k\}$ for all $p\in\Ccal_k$, then $\phi_k$ is independent of $\tau_{-k}$. Consequently, $\mu_{k,i}$ is constant with respect to its second argument, yielding $H_k=0$. In this case, DMBD exactly reduces to each robot independently running MBD on its own subproblem.

\begin{remark}
The bound in~\Cref{prop:score_bound} is stated for the exact score functions. In practice, however, both MBD and DMBD estimate these scores using finite sets of samples (as in~\eqref{eq:gradient} and~\eqref{eq:local_gradient}). Consequently,~\eqref{eq:score_bound} characterizes the asymptotic behavior as the number of particles satisfies $S\to\infty$ and $S_k\to\infty$. With finitely many samples, the estimated scores may incur additional Monte Carlo error; a finite-sample analysis is left to future work.
\end{remark}

\begin{remark}
    \Cref{prop:score_bound} depends on the Lipschitz assumption of the local denoiser $\mu_{k,i}(\cdot, z)$ with respect to $z$. Such assumption can be satisfied when $p_0^k(\cdot\mid z)$ varies Lipschitz continuously in $1$-Wasserstein distance~\cite[Remark 6.5]{villani2009optimal}. 
\end{remark}

\section{Simulations}
\label{sec:sim}

In this section, we present simulation results to answer two questions: \textbf{Q1.} Is DMBD scalable with respect to the number of robots $N$? \textbf{Q2.} Can robots generate trajectories that minimize the global cost without global knowledge (e.g., global dynamics, objective, and constraints)? All simulations are coded in Python with JAX~\cite{jax2018github} for GPU acceleration and were conducted on a computer with a 12th Gen Intel® Core™ i9-12900KF CPU, 64 GB of RAM, and an Nvidia RTX 3080 Ti GPU.  While our algorithm can be used for online trajectory planning, we focus on offline planning only.

We consider two systems. First, double integrator dynamics with state
$\xbf_t^k = \begin{bmatrix}
\mathbf{p}_t^{k\top} & \mathbf{v}_t^{k\top}
\end{bmatrix}^\top$ and control $\ubf_t^k = \begin{bmatrix}
    a_{x,t}^k & a_{y,t}^k
\end{bmatrix}^\top$, where $\mathbf{p}_t^k = \begin{bmatrix}
    x_t^k & y_t^k
\end{bmatrix}^\top$ and $\mathbf{v}_t^k = \begin{bmatrix}
    v_{x,t}^k & v_{y,t}^k
\end{bmatrix}^\top$. 
The system evolves as
{\small
\eqn{
\label{eq:double_int}
\begin{aligned}
\mathbf{p}_{t+1}^k = \mathbf{p}_t^k + \mathbf{v}_t^k \Delta t + \tfrac{1}{2}\ubf_t^k \Delta t^2,\ \ 
\mathbf{v}_{t+1}^k = \mathbf{v}_t^k + \ubf_t^k \Delta t,
\end{aligned}
}}with $\|\ubf_t^k\|_\infty\leq 1$ and $\|\mathbf{v}_t^k\|_\infty \leq 5$, where $\Delta t=0.1$ is the sampling time. Second, we consider kinematic bicycle dynamics~\cite{polack2017kinematic} with state 
$\xbf_t^k=\begin{bmatrix} x_t^k & y_t^k & \theta_t^k & v_t^k \end{bmatrix}^\top$ and control $\ubf_t^k=\begin{bmatrix} a_t^k & \delta_t^k \end{bmatrix}^\top$, subject to $|v_t^k|\leq 5$, $|a_t^k|\leq 1$, and $|\delta_t^k|\leq 0.25$. The dynamics are
{\small
\seqn[\label{eq:kinematic}]{
x_{t+1}^k &= x_t^k + v_t^k \cos(\theta_t^k)\Delta t, \label{eq:kinematic_1}\\
y_{t+1}^k &= y_t^k + v_t^k \sin(\theta_t^k)\Delta t, \label{eq:kinematic_2}\\
\theta_{t+1}^k &= \theta_t^k + \frac{v_t^k}{L^k}\tan(\delta_t^k)\Delta t, \label{eq:kinematic_3}\\
v_{t+1}^k &= v_t^k + a_t^k \Delta t,\label{eq:kinematic_4}
}}where  $\Delta t=0.25$ is the sampling time and $L^k$ is the wheelbase length of agent $k$.

We consider circular and rectangular robots. A circular robot $k$ is modeled by its radius $R_{{\rm col}}^k$, and collision avoidance between robots $k$ and $j$ is enforced through

\eqnN{h^{k,j}_{\rm cir}(\xbf^k_t, \xbf^j_t)=\max\{0, (R_{{\rm col}}^k+R_{{\rm col}}^j)^2 - \|\mathbf{p}^k_t-\mathbf{p}^j_t\|_2^2\}.}

Collision checks among rectangular robots with dimensions $d_{1,k}\times d_{2,k}$ are handled using the separating axis theorem (SAT)~\cite{ericson2005real}, denoted
\eqn{h^{k,j}_{\rm rec}(\xbf^k_t, \xbf^j_t)= \Psi_{\rm SAT}(\xbf^k_t, \xbf^j_t),}
where $\Psi_{\mathrm{SAT}}
(\mathbf{x}_t^k,\mathbf{x}_t^j)= 0$
indicates that robots $k$ and $j$ do not overlap. We want to enforce these constraints for all $t\in\Tcal:=\{0,\dots,T\}$. Thus, the robot $k$ has a collision constraint $p\in \Ccal_k$ with robot $j\in \Rcal_{k,p}=\{k,j\}$ as 
\begin{equation}
\label{eq:collision_constraint}
g_{k,p}(Y)
=
\begin{cases}
\sum_{t\in\Tcal} h_{\mathrm{cir}}^{k,j}(\mathbf{x}_t^k,\mathbf{x}_t^j),
\ \text{circular robots},\\
\sum_{t\in\Tcal}  h_{\mathrm{rec}}^{k,j}(\mathbf{x}_t^k,\mathbf{x}_t^j),
\ \text{rectangular robots}.
\end{cases}
\end{equation}

To enforce collision avoidance with the static circular obstacles $\Ocal_q$, we define the obstacle collision constraint $p\in \Ccal_k$ analogously to~\eqref{eq:collision_constraint}. Specifically,
\begin{equation}
\label{eq:obstacle_constraint}
g_{k,p}(Y)=
\begin{cases}
\sum_{t\in\Tcal} h_{\mathrm{cir}}^{k,q}(\mathbf{x}_t^k,\Ocal_q),
\ \text{circular robots},\\
\sum_{t\in\Tcal} h_{\mathrm{rec}}^{k,q}(\mathbf{x}_t^k,\Ocal_q),
\ \text{rectangular robots},
\end{cases} \nonumber
\end{equation}
with $h_{\mathrm{cir}}^{k,q}$ denoting the distance-based circle-to-circle collision test and
$h_{\mathrm{rec}}^{k,q}$ the closest-point projection test between an oriented rectangle and a circular obstacle~\cite{ericson2005real}.

\textbf{Q1. Scalability:}
To answer \textbf{Q1}, we compare DMBD against four SBO baselines across different environments: Cross-Entropy Method (CEM)~\cite{botev2013cross}, Model Predictive Path Integral (MPPI)~\cite{williams2018information}, MBD~\cite{pan2024model}, and D4orm~\cite{zhang2025d4orm}. All methods
\begin{itemize}
\item optimize trajectories in the control-input space with the same planning horizon $T$ and denoising steps (iterations for CEM and MPPI) $M$,
\item sample 500 candidate control trajectories \emph{per robot} to match the total number of samples, and
\item use cost parameters tuned once using the MBD on the case with $N=2$ for each scenario.
\end{itemize}
For CEM, we use a 2\% elite set (10 samples per robot), while D4orm uses two deformation iterations ($2M$ denoising steps). To evaluate them, we measure each method using two metrics: (i) the \textit{success rate}, defined as the percentage of trials that produce a collision-free solution, and (ii) the \textit{planning time}, measured as the time required to generate a trajectory from the initial states $\xbf_0$.

\begin{figure}
    \centering
    \includegraphics[width=\linewidth]{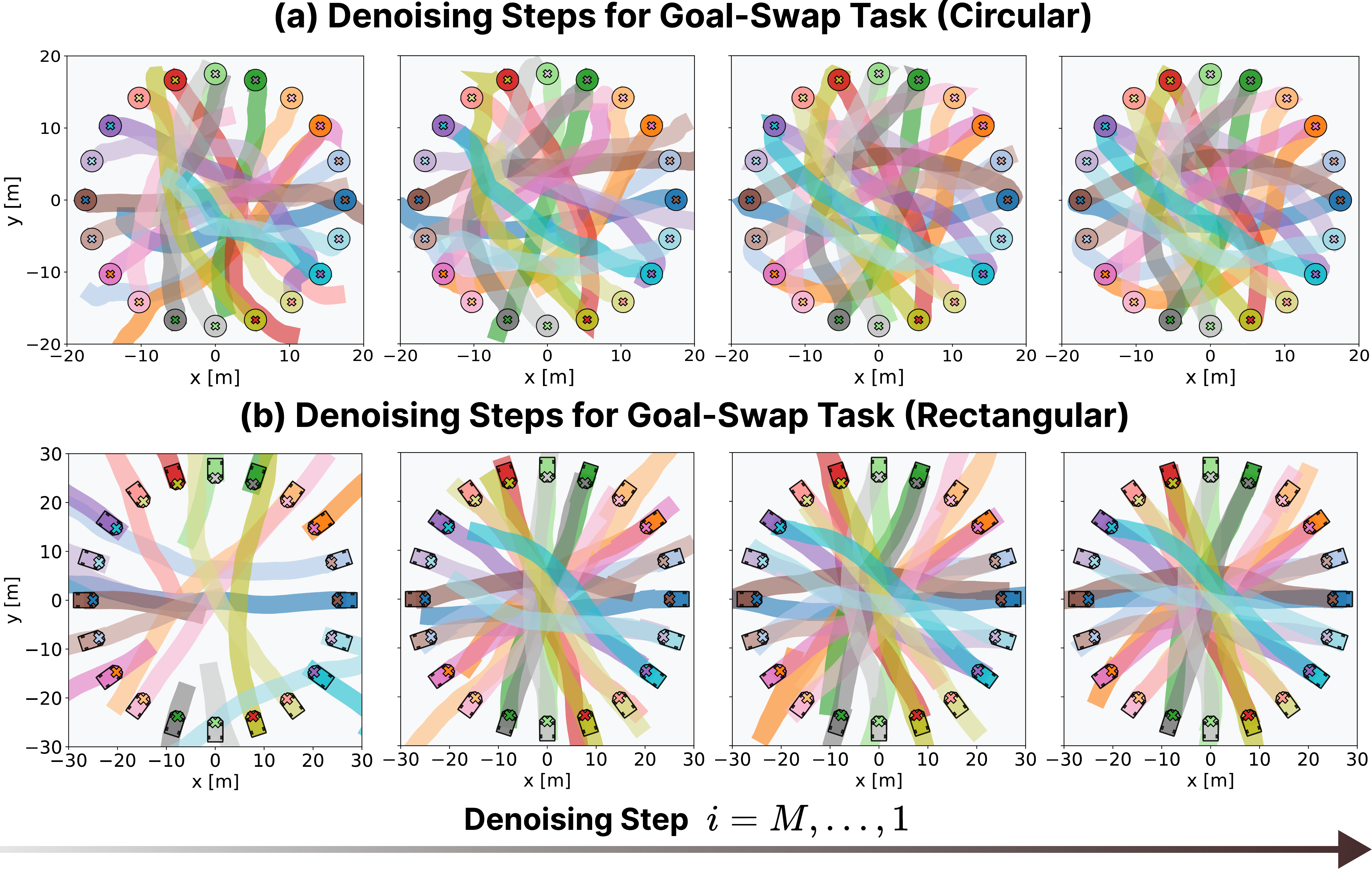}
    \caption{DMBD's local conditional denoising process for goal swapping with $N=20$ (a) circular and (b) rectangular robots.}
    \label{fig:goal_swaps}
\end{figure}

\subsubsection{\textbf{Goal-Swap}}
We evaluate a circular goal-swapping task with (i) circular robots of radius $R_{{\rm col}}^k=1.5\,\mathrm{m}$ under~\eqref{eq:double_int} and (ii) rectangular robots of size $3.0\times5.0$ under~\eqref{eq:kinematic} with $L^k=3.0$. Robots are initialized uniformly on circles of radius $17.5\,\mathrm{m}$ and $25.0\,\mathrm{m}$, respectively, with diametrically opposite goals (shown in~\Cref{fig:goal_swaps}). We consider $N\in\{2,4,\dots,20\}$, set $M=150$, and increase the horizon from $T=200$ by 5 steps per additional two robots. A trial succeeds if all robots remain collision-free and reach within $2.5\,\mathrm{m}$ of their goals.  We run 50 trials with varying seeds.

\begin{figure}
    \centering
    \includegraphics[width=0.95\linewidth]{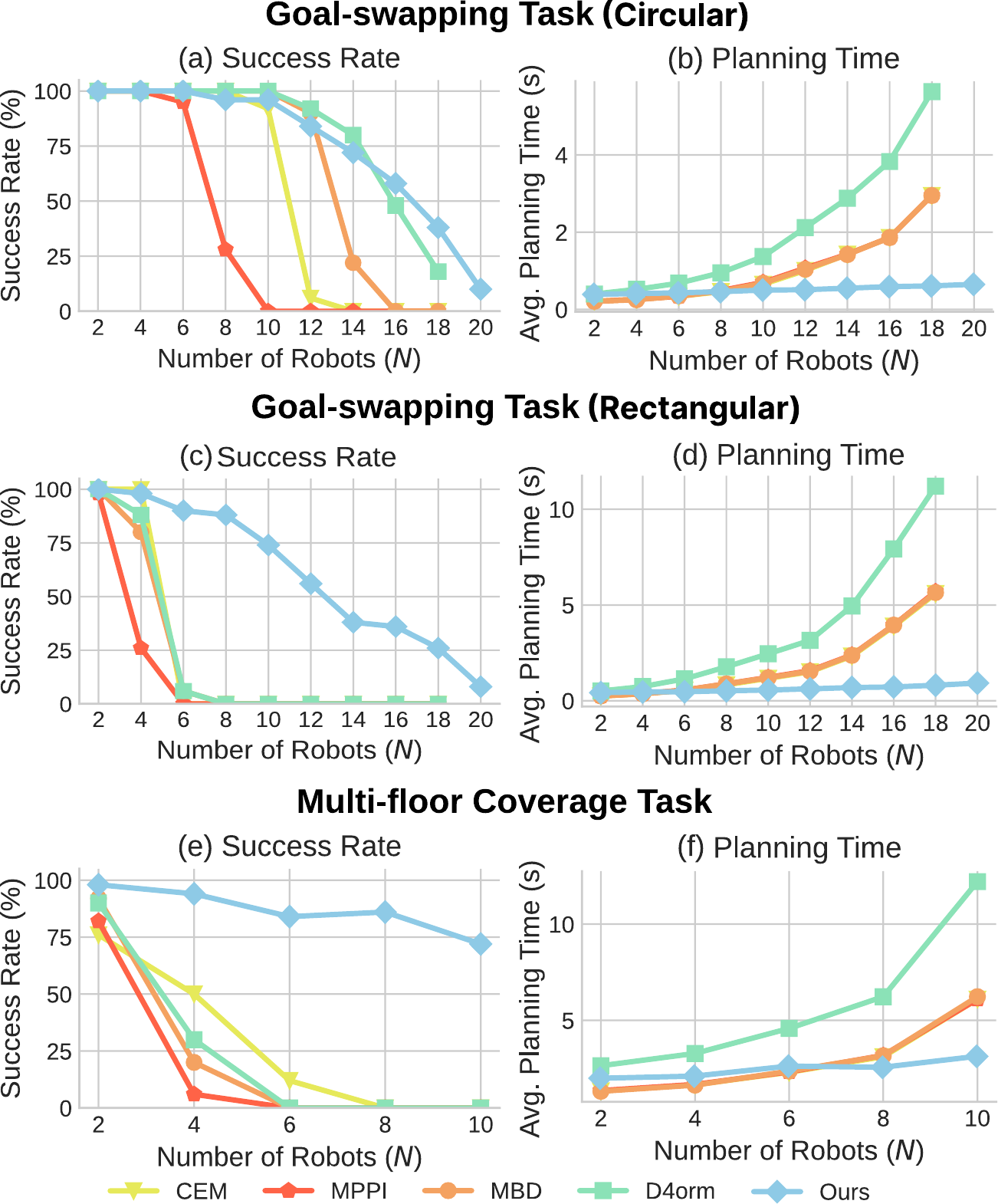}
    \caption{Success rates and average planning time across $50$ trials for varying $N$. Baseline methods fail at $N=20$ due to out-of-memory errors for goal-swapping tasks.}
    \label{fig:comparison}
\vspace{-15pt}
\end{figure}



\begin{figure*}
    \centering
    \includegraphics[width=0.85\linewidth]{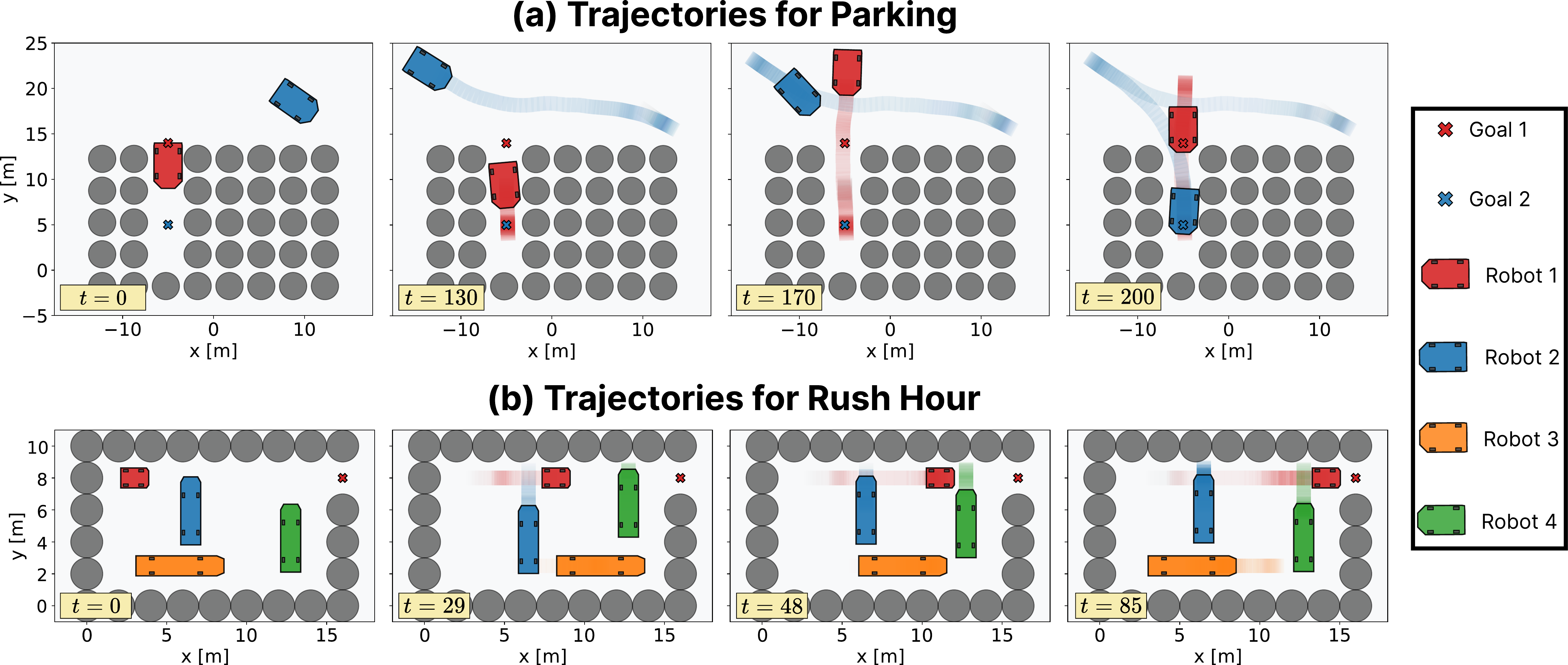}
\caption{Visualizations of representative planned trajectories through DMBD for the (a) parking and (b) rush hour scenarios.}
    \label{fig:second_sim}
\vspace{-15pt}
\end{figure*}

\subsubsection{\textbf{Multi-Floor Coverage}}
We consider a heterogeneous team of rectangular robots navigating a two-story building. The team consists of smaller robots $\mathcal{E}=\{1,\dots,N/2\}$ with dimensions $2.0\times4.5$ and larger robots $\Rcal\setminus\mathcal{E}$ with dimensions $2.5\times6.5$. All robots follow the kinematic model~\eqref{eq:kinematic} with augmented state $\xbf_t^k=\begin{bmatrix}x_t^k&y_t^k&z_t^k&\theta_t^k&v_t^k\end{bmatrix}^\top$, where $z_t^k$ denotes the floor height. The wheelbase is $L^k=3.0$ for $k\in\mathcal{E}$ and $L^k=4.2$ otherwise. Smaller robots may move up to the second floor ($z=5$) only through four elevator zones, whereas larger robots remain on the first floor ($z=0$). Eight cylindrical obstacles of radius $1.0\,\mathrm{m}$ are placed at $(\pm20,0,z)$ and $(0,\pm20,z)$, $z\in\{0,5\}$.

We evaluate team sizes $N\in\{2,4,\dots,10\}$ with $M=200$ denoising steps and increase the horizon from $T=250$ by 25 steps per additional two robots. A trial succeeds if all robots remain collision-free and reach their assigned goals within $5.0\,\mathrm{m}$ on the correct floor. We run each configuration 50 times with randomly generated initial states $(x_0^k,y_0^k,z_0^k)\in[-30,30]^2\times\{0\}$, $\theta_0^k\in[-\pi,\pi]$, and $v_0^k=0$, and goals in $[-20,20]^2\times\{0,5\}$, where smaller and larger robots have $z=5$ and $z=0$, respectively.

\subsubsection{\textbf{Results and Discussions}} 

As shown in~\Cref{fig:comparison}, DMBD in general attains the highest success rate in the rectangular goal-swapping and multi-floor coverage tasks while maintaining near-constant planning time as $N$ increases. In contrast, baselines suffer from rapidly increasing planning time and decreasing success rate. D4orm performs comparably to DMBD in the circular-robot goal-swapping scenario but requires approximately 10$\times$ longer planning time. All baselines also eventually fail at $N=20$ due to out-of-memory (OOM) errors when sampling from the high-dimensional joint trajectory space.

This scalability advantage arises because baseline methods optimize over the coupled trajectory space of dimension $(\sum_{k=1}^N m_k)T$, which grows with team size, whereas DMBD decomposes inference into per-robot trajectory spaces of dimension $m_kT$. Furthermore, the performance gap widens as task complexity increases, from circular robots to rectangular robots and multi-floor coverage, where rectangular-body collision constraints and elevator dynamics introduce additional non-smoothness. These results highlight the advantage of DMBD for challenging coordination problems.

\begin{table}[t]
\centering
\caption{Success Rate and Planning Time Across Different Tasks}
\label{tab:difficult_cases}

\small
\renewcommand{\arraystretch}{1.15}

\begin{tabular}{lcc}
\toprule
\textbf{Metric}
& \textbf{Parking} 
& \textbf{Rush Hour} \\

\midrule
 \begin{tabular}[c]{@{}c@{}}\textbf{Success (\%)} \\ \textbf{Planning Time (s)}\end{tabular} & \begin{tabular}[c]{@{}c@{}} 92 \\ 0.52 \end{tabular} & \begin{tabular}[c]{@{}c@{}} 92 \\ 0.39 \end{tabular}  \\ 
\midrule
\end{tabular}
\vspace{-15pt}
\end{table}

\textbf{Q2. Coordination without global knowledge:} To answer \textbf{Q2}, we consider two challenging environments to evaluate whether robots can cooperatively solve the global optimization problem without access to global information, including global objectives, dynamics, and constraints.
\setcounter{subsubsection}{0}
\subsubsection{\textbf{Parking}} We consider a parking scenario with two rectangular robots under the kinematic bicycle model~\eqref{eq:kinematic}, shown in~\Cref{fig:second_sim}(a). To reach their designated parking spots, robot~1 (red) must temporarily leave its goal location, allowing robot~2 (blue) to park, before returning. A trial is considered successful if both robots reach within $2.5\,\mathrm{m}$ of their goal locations without collisions. We run 50 trials with robot~2's initial position sampled uniformly from $(x_0^k,y_0^k)\in[-15,15]\times[17.5,25]$ and $\theta_0^k\in [0,2\pi]$ with $v_0^k=0$.

\subsubsection{\textbf{Rush Hour}} We consider a four-robot ``Rush Hour'' puzzle in which each rectangular robot is restricted to longitudinal motion with state $\xbf_t^k= \begin{bmatrix}
    x_t^k & y_t^k&\theta_t^k&v_t^k
\end{bmatrix}^\top$ and dynamics~\eqref{eq:kinematic_1},~\eqref{eq:kinematic_2}, and~\eqref{eq:kinematic_4}. As illustrated in~\Cref{fig:second_sim}(b), robot~1 must exit the environment, but its path is blocked by other robots whose motions are mutually constrained. Robots~2-4 seek only to minimize their own control effort. A trial is successful if robot~1 reaches within $1.5\,\mathrm{m}$ of its goal without collisions. We evaluate 50 trials with randomized initial longitudinal positions of the robots 2-4.

\subsubsection{\textbf{Results and Discussions}} 

These tasks require robots to coordinate and execute strategically suboptimal local actions without knowing others' goals, while considering highly non-convex and non-smooth local objectives and constraints. Nevertheless, DMBD achieves over $90\%$ success rates across both tasks within sub-second planning times (Table~\ref{tab:difficult_cases}). The resulting coordination emerges without sharing the global objectives, demonstrating that local conditional denoising enables the robots to recover globally coordinated behavior.

\section{Conclusions}

\label{sec:conc}
We present Distributed Model-Based Diffusion (DMBD), a scalable server-robot framework for multi-robot trajectory optimization. By distributing the reverse diffusion process, DMBD enables each robot to perform local conditional denoising using only local information while conditioning on the trajectory estimates from other robots. We characterize the relationship between the global and local conditional denoising updates by deriving a theoretical error bound at each denoising step. Extensive simulations across diverse scenarios demonstrate that DMBD outperforms existing baselines.





\bibliographystyle{IEEEtran}
\bibliography{references_ll}

\end{document}